\documentclass[english]{article}
\usepackage[T1]{fontenc}
\usepackage[latin9]{inputenc}
\usepackage{amsmath}
\usepackage{amsthm}
\usepackage{amssymb}
\usepackage{graphicx}
\usepackage{geometry}
\makeatletter
\theoremstyle{definition}
 \newtheorem{example}{\protect\examplename}
\theoremstyle{plain}
\newtheorem{assumption}{\protect\assumptionname}
\theoremstyle{remark}
\newtheorem{rem}{\protect\remarkname}
\theoremstyle{plain}
\newtheorem{thm}{\protect\theoremname}

\ifdefined\showcaptionsetup
 \PassOptionsToPackage{caption=false}{subfig}
\fi
\usepackage{subfig}
\makeatother

\usepackage{babel}
\providecommand{\assumptionname}{Assumption}
\providecommand{\examplename}{Example}
\providecommand{\remarkname}{Remark}
\providecommand{\theoremname}{Theorem}

\begin{document}
\title{Polynomial Speedup in Diffusion Models\\
with the Multilevel Euler-Maruyama Method}
\author{Arthur Jacot}
\maketitle
\begin{abstract}
We introduce the Multilevel Euler-Maruyama (ML-EM) method compute
solutions of SDEs and ODEs using a range of approximators $f^{1},\dots,f^{k}$
to the drift $f$ with increasing accuracy and computational cost,
only requiring a few evaluations of the most accurate $f^{k}$ and
many evaluations of the less costly $f^{1},\dots,f^{k-1}$. If the
drift lies in the so-called Harder than Monte Carlo (HTMC) regime,
i.e. it requires $\epsilon^{-\gamma}$ compute to be $\epsilon$-approximated
for some $\gamma>2$, then ML-EM $\epsilon$-approximates the solution
of the SDE with $\epsilon^{-\gamma}$ compute, improving over the
traditional EM rate of $\epsilon^{-\gamma-1}$. In other terms it
allows us to solve the SDE at the same cost as a single evaluation
of the drift. In the context of diffusion models, the different levels
$f^{1},\dots,f^{k}$ are obtained by training UNets of increasing
sizes, and ML-EM allows us to perform sampling with the equivalent
of a single evaluation of the largest UNet. Our numerical experiments
confirm our theory: we obtain up to fourfold speedups for image generation
on the CelebA dataset downscaled to $64\times64$, where we measure
a $\gamma\approx2.5$. Given that this is a polynomial speedup, we
expect even stronger speedups in practical applications which involve
orders of magnitude larger networks.
\end{abstract}

\section{Introduction}

Denoising Diffusion Probabilistic Models (DDPMs) \cite{sohldickstein2015unsupervised,ho2020_DDPM,song2021scorebased}
are the state of the art technique for generating images \cite{rombach2022_latent_diff_models},
videos, and many more \cite{watson2023novo,arts2023moleculardynamics}.
The images are generated by a diffusion process, a Stochastic Differential
Equation (SDE), or sometimes an ODE which starts from Gaussian noise
and ends up with a distribution that approximates the `true data distribution'.
The drift term is learned with a large Deep Neural Network (DNN) -
typically a UNet, a type of Convolutional Neural Network (CNN) - so
that the computational cost of DDPMs is dominated by the number of
DNN evaluations (sometimes written NFE: ``number of function evaluations''),
which has to increase to reach higher levels of accuracy, or equivalently
a smaller error $\epsilon$ between the continuous SDE solution and
the chosen discretization.

This high computational cost has a large environmental impact given
the wide spread use of these models and the large scale of the underlying
DNN. It also limits the application of these techniques to setting
that require real-time generation such as music. Several methods have
been proposed to reduce this computational cost:

The SDE noise forces relatively small step sizes, so \cite{song2020_DDIM}
have proposed replacing the backward SDE with a backward ODE (or something
in between, a SDE with a smaller Brownian term) called the Denoising
Diffusion Implicit Models (DDIM). DDIMs can achieve realistic images
with an order of magnitude less steps than DDPMs, thus drastically
reducing the NFE. Assuming perfect estimation of the score, the final
distribution should be the same for DDPMs and DDIMs, in practice however,
it appears that DDIMs result in slightly lower quality images, even
at the smallest step sizes. Several other methods have been proposed
to further `straighten' the flow \cite{liu2022_rectified_flow,lipman2022_flow_matching},
or reduce the dimensionality by working in a latent space \cite{rombach2022_latent_diff_models}.
The ODE formulation also opens the door to adapting efficient solvers
such as the Runge-Kutta family of methods to DDIMs \cite{lu2022_DPM_solver,lu2022_DPM_solver_++}.
Yet another approach is to train a new DNN to implement multiple steps
of the denoising process, or even the full denoising process \cite{salimans2022_progressive_distillation_diffusion,meng2023_guided_distillation,sauer2023_adversarial_diff_distil,song2023consistency}.

\subsection{Compute Scaling Analysis\label{subsec:Compute-Scaling-Analysis}}

We start with a simple `napkin math' analysis of the scaling $T\sim\epsilon^{-\alpha}$
of the compute time $T$ required to generate an image with error
$\epsilon$. The error $\epsilon$ can be decomposed into a discretization
error $\epsilon_{discr}$ and approximation error $\epsilon_{approx}$. 

\textbf{Discretization error: }We expect the discretization error
$\epsilon_{discr}$ to scale as $n^{-\frac{1}{\phi}}$ where $n$is
the number of steps or NFE. For example, one has $\phi=1$ with the
Euler-Maruyama method, and $\phi=\frac{1}{4}$ for 4th order Runge-Kutta
method.

\textbf{Approximation error:} We expect the approximation error $\epsilon_{approx}$
between our DNN and the `true denoiser' (or equivalently the true
score) to follow a typical scaling law $\epsilon_{approx}=N^{-\frac{1}{\psi}}+P^{-\frac{1}{\gamma}}$
for $N$ the number of training samples and $P$ the number of parameters,
matching what has been observed empirically \cite{kaplan2020scaling,henighan2020scaling,hoffmann2022scaling}.
Since the training data size is irrelevant to the computational cost
of generating images, we drop the data-size term $N^{-\frac{1}{\psi}}$
(i.e. we assume that we always have enough data that the approximation
error is dominated by the network size, and not the dataset size).
We will consider the rate $\gamma$ to be given, but there exists
a range of theoretical works that give predictions for $\gamma$ under
a number of settings \cite{bahri2024explaining,Bordelonetal2024b,Bordelonetal2025,Michaudetal2023}.

Assuming that the computational cost of a DNN is proportional to $P$,
it is optimal to choose $n\sim\epsilon^{-\phi}$ and $P\sim\epsilon^{-\gamma}$
to reach an error of $\epsilon$ at a computational cost of order
$nP\sim\epsilon^{-(\phi+\gamma)}$. This leads to a scaling law of
$\epsilon^{-(\gamma+1)}$ for the Euler-Maruyama method and $\epsilon^{-(\gamma+\frac{1}{4})}$
for Runge-Kutta or a variant thereof. Most of the aforementioned methods
should only lead to constant improvements, with no improvements on
the exponents, except for the improved solvers inspired by Runge-Kutta,
as already mentioned. In this framework, it seems that a rate of $\epsilon^{-\gamma}$
is impossible, because any discretization method would require a growing
NFE to reach smaller and smaller errors.

This paper shows that a rate of $\epsilon^{-\gamma}$ is actually
possible if we assume that the score is hard enough to approximate
$\gamma>2$. Inspired by Multilevel Monte Carlo (MLMC) methods, we
rely on a Multilevel Euler-Maruyama method, where a range of DNNs
of increasing sizes are randomly used at each discretization step,
using large DNNs with low probability and small ones with high probability,
thus leading to a small overall computational cost. This is especially
impactful for SDEs, where no fast discretization methods exist (i.e.
no Runge-Kutta or analogue, with $\phi<1$).

It might seem almost paradoxical how assuming that the task is ``hard
enough'' allows us to gain speedups. But previous work has shown
the emergence of beneficial properties such as convexity in the so-called
Harder than Monte Carlo (HTMC) regime (when $\gamma>2$) \cite{jacot2025_DLCPC}.
This same paper \cite{jacot2025_DLCPC} also proves a connection between
the sets of functions approximable with a DNN and different HTMC spaces,
which motivates our HTMC assumption ($\gamma>2$). This assumption
is further motivated by empirical evidence that DNNs trained on images
follow scaling laws \cite{henighan2020scaling}, which obtains empirical
rates $\epsilon^{2}\sim P^{-0.24}$ for $8\times8$ images, $P^{-0.22}$
for $16\times16$, and $P^{-0.13}$ for $32\times32$ images, which
would correspond to $\gamma\approx8.3$, $\gamma\approx9.1$, and
$\gamma\approx15.4$ respectively. These are all far into the HTMC
regime, and the bigger the image the larger the rate $\gamma$. This
paper is just a first example of the kind of speedups that can be
obtained under the HTMC assumption.

\section{Setup}

This paper focuses on the efficient approximation of SDEs of the form
\[
dx_{t}=f_{t}(x_{t})dt+\sigma_{t}dW_{t}.
\]
For simplicity, we assume that the noise is isotropic and that its
variance $\sigma_{t}^{2}$ depends on time $t$ but not on $x_{t}$.
We then consider ODEs as the special case $\sigma_{t}=0$.
\begin{example}[Denoising Diffusion Probabilistic Model - DDPM]
\textbf{ }The main motivation is to apply it to the reverse diffusion
process (which starts from a large $T>0$ and then goes back in time
until reaching $t=0$, hence the minus in front of $dx_{t}$).
\[
-dx_{t}=\left(\frac{1}{2}x_{t}+s_{t}(x_{t})\right)dt+dW_{t}
\]
for the score $s_{t}(x_{t})=\nabla\log\rho_{t}$ where $\rho_{t}$
is the distribution of $\sqrt{e^{-t}}x_{0}+\sqrt{1-e^{-t}}\mathcal{N}(0,1)$.
\end{example}
\begin{example}[Denoising Diffusion Implicit Model - DDIM]
\textbf{} We will also consider the probability flow ODE, which has
the same marginal distribution at each time $t$, but has no noise
term:
\[
-\frac{dx}{dt}=\frac{1}{2}x_{t}+\frac{1}{2}s_{t}(x_{t}).
\]
\end{example}
What makes the diffusion SDE unique is the fact that the score $s_{t}$
can only be approximated by very large DNNs, and these DNNs generally
follow scaling laws, i.e. the size of the network (and therefore its
computational cost) must scale rapidly if one wants to obtain more
and more accurate approximations of the true score. We formalize this
into the following assumption:
\begin{assumption}[Scaling Law Assumption]
\label{assu:scaling_law}There is a sequence of estimators $f_{t}^{k}$
which approximate $f_{t}$ within a $2^{-k}$ error
\[
\left\Vert f_{t}-f_{t}^{k}\right\Vert _{\infty}\leq2^{-k}
\]
for all $t\in[0,T]$ and whose compute $C(f_{t}^{k})$ scales exponentially
in $k$:
\[
C(f_{t}^{k})\leq c^{\gamma}2^{\gamma k}
\]
for some $\gamma$ (the convention of taking the constant $c$ to
the $\gamma$-th power will lead to cleaner formulas).
\end{assumption}
This assumption is motivated by the strong empirical evidence that
DNNs follow scaling laws \cite{kaplan2020scaling,henighan2020scaling}:
there is a $\gamma$ such that the test error can be bounded $\sqrt{\mathbb{E}\left\Vert f(x)-f_{\theta}(x)\right\Vert ^{2}}\leq cP^{-\frac{1}{\gamma}}$
in terms of the number of parameters $P$ of the DNN, a scaling constant
$\gamma$ and a prefactor $c$. Roughly speaking, for fully-connected
networks, the computational cost of evaluating $f_{\theta}$ is proportional
to $P$, since each parameter is used once, for CNNs it scales as
$Pwh$ (for $w,h$ the width and height of the image), for RNNs as
$P\ell$ (where $\ell$ is the sequence length), and for Transformers
as $P\ell^{2}$. We can then rewrite the bound to match the scaling
assumption: $C(f_{\theta})\leq sc^{\gamma}\epsilon^{-\gamma}$ (where
$s$ is either $1,wh,\ell$ or $\ell^{2}$). The ubiquity of these
scaling laws in practice implies that this is a very reasonable assumption
(note that we assume a bound on the $L_{\infty}$ rather than the
$L_{2}$ error, but this is mainly to simplify the derivations, a
bound on the $L_{2}$ could be shown to be enough with a few extra
assumption and a bit more work).

The constant $c$ in the assumption is closely related to the so-called
HTMC norm $\left\Vert f\right\Vert _{M^{\gamma}}$ as defined in \cite{jacot2025_DLCPC}:
if for all $k$ the estimator $f_{t}^{k}$ has minimal computational
complexity amongst all $2^{-k}$-estimators (in the sense that it
minimizes circuit size), then $c=\left\Vert f\right\Vert _{M^{\gamma}}$.
However we do not need to assume that we have found the most computational
efficient estimator for the results of this paper to apply, and thus
in general we only have $c\geq\left\Vert f\right\Vert _{M^{\gamma}}$.

\textbf{Euler-Maruyama Method:} The baseline we consider to approximate
our SDE algorithmically is the Euler-Maruyama method \cite{maruyama1955continuous}
together with a certain approximation $f^{k}$ of $f$ (typically
this would be the largest DNN we can train), yielding
\[
y_{t+\eta}=y_{t}+\eta f_{t}^{k}(y_{t})+\sqrt{\eta}\sigma_{t}Z_{t}
\]
with a step size $\eta$ (which we assume constant) and $Z_{t}\sim\mathcal{N}(0,1)$.
In the absence of noise ($\sigma_{t}=0$), we recover the Euler method.
\begin{rem}
Note that for DDPM, the Euler-Maruyama discretization is slightly
different from the usual implementation of DDPM, and similarly for
DDIM. We describe in Appendix \ref{sec:appendix-comparison} why the
two are equivalent up to subdominant terms as the learning rate goes
to zero.

\begin{figure}
\subfloat[DDPM MSE]{\includegraphics[scale=0.54]{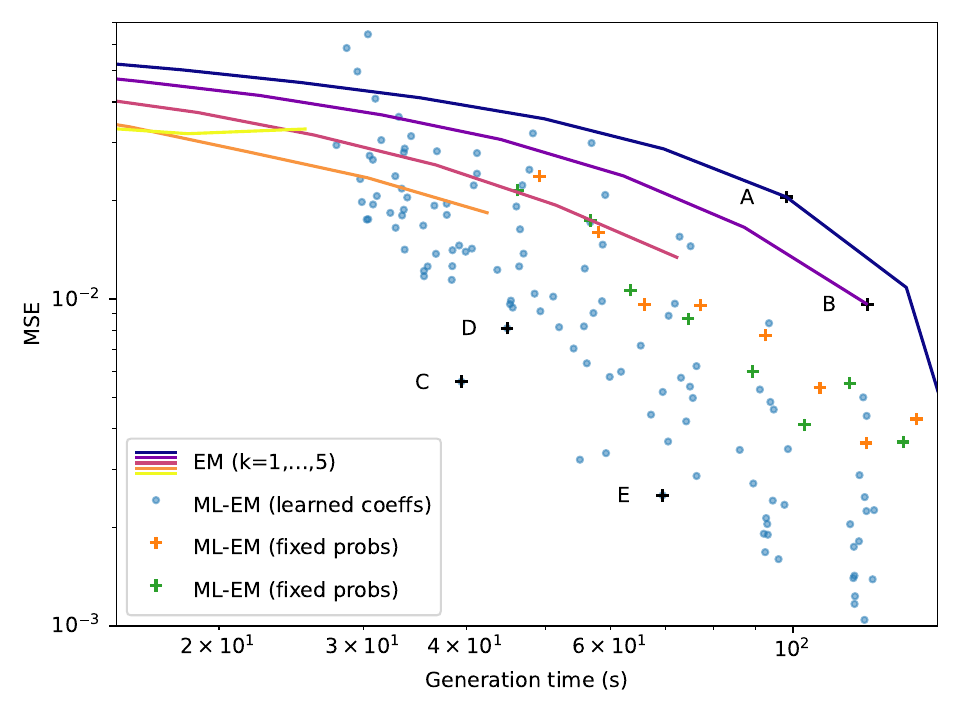}

}\subfloat[DDPM samples]{\includegraphics[scale=0.45]{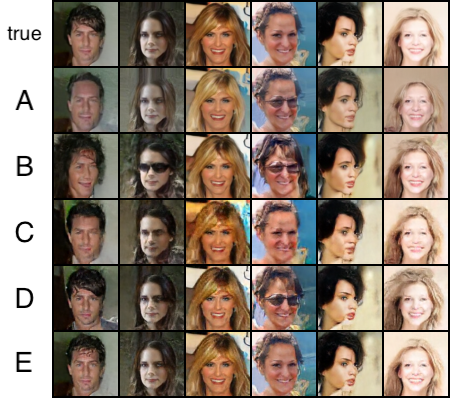}

}

\subfloat[DDIM MSE]{\includegraphics[scale=0.58]{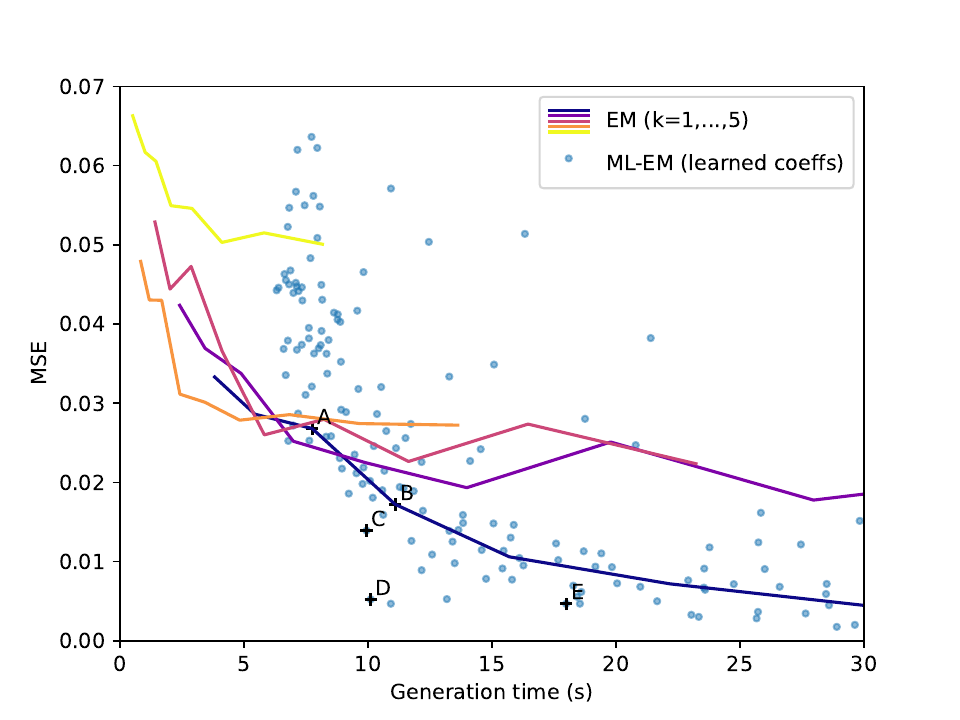}

}\subfloat[DDIM samples]{\includegraphics[scale=0.45]{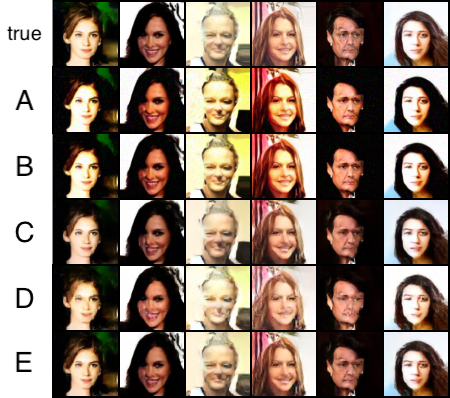}

}

\caption{(Left) We compare ML-EM to EM method of generation for DDPM (top)
and DDIM (bottom) by plotting the MSE between the generated sample
and the `true' sample (generated with a 1000 steps DDPM/DDIM) with
the same initial and Brownian noise, the $x$-axis is the time in
seconds required to generate 200 images. Solid lines are the traditional
EM method with different network sizes $f^{1},\dots,f^{5}$ and with
number of steps ranging from $58$ to $933$. The crosses and dots
are the ML-EM method with three networks $\{f^{1},f^{3},f^{5}\}$
and with either fixed probabilities or learned coefficients $\alpha_{k},\beta_{k}$
(see Section \ref{sec:Numerical-Experiments}). We add a $\Delta\in\{-3.0,-2.5,\dots,2.5,3.0\}$
to the $\beta_{k}$s and perform 15 trials over the sampling of the
Bernoullis RVs (remember that the starting noise and Brownian motion
are fixed). The sampling of the Bernoullis that yield the smallest
MSE can be memorized, it is therefore okay to compare the straight
lines of classical EM to the best trials of ML-EM. (Right) The first
6 generated images for the `true sample' and four selected instances
of EM (A,B) and ML-EM (C,D,E). \protect \\
For DDPMs, ML-EM with learned coefficients clearly outperforms all
other methods, requiring in some cases 4 times less compute time than
EM to reach the same MSE, or reaching a 10 times smaller MSE at the
same compute time. For DDIM the advantage of ML-EM is less clear,
but still present. Visually, it appears that the main advantage of
ML-EM is that it avoids discolorations/contrast issues present for
EM with few steps. Interestingly, DDIM appears to suffer from these
discoloration even with 1000 steps.}
\end{figure}
\end{rem}

\section{Multilevel Euler-Maruyama\label{sec:ML-EM}}

Our strategy is to use a variation on the Multilevel Monte Carlo (MLMC)
\cite{Giles2008MultilevelMC,Giles2015} method at each step of the
discretization\footnote{The original motivation for MLMC was to compute expectations over
the sampling of SDEs, in which case multiple discretization of SDE
paths are computed, with different step-sizes to obtain different
levels of accuracy. In our case, we only want to evaluate one SDE
path, and the multiple levels result from the different DNN sizes.
This similarity in the setting might lead to confusion, but MLMC is
not specific to SDEs, it can be applied whenever one has access to
a range of estimators with different errors and computational cost.}:
\[
y_{t+\eta}=y_{t}+\eta\sum_{k=k_{min}}^{k_{max}}\frac{B^{k}}{p_{k}}\left[f_{t}^{k}(y_{t})-f_{t}^{k-1}(y_{t})\right]+\sqrt{\eta}\sigma_{t}Z_{t}
\]
where $B^{k}\sim\text{Bernoulli}(p_{k})$. The idea is that we are
going to choose a probability $p_{k}$ that decreases exponentially
in $k$ so that at most steps, we will not need to evaluate the best
estimator $f^{k_{max}}$. Note that our guarantees will therefore
be in terms of the expected computational cost 
\[
\mathbb{E}C(y_{T})=\sum_{t,k}p_{k}C(f_{t}^{k})\leq\frac{T}{\eta}c^{\gamma}\sum_{k}p_{k}2^{\gamma k}.
\]
One could then use the probabilistic method to imply the existence
of a deterministic choice of the $B^{k}$ that reaches a certain error
at a certain computational cost. In practice, we observe that $C(y_{T})$
concentrates in its expectation, whereas the error exhibits a significant
variance over the sampling of the $B^{k}$ (though it is very consistent
across different initialization of the SDE and the sampling of the
Brownian motion). We therefore perform a best of 15 to identify the
optimal choices of Bernoulli random variables $B^{k}$.

We choose $k_{min}(t)=-\left\lceil \log_{2}\left\Vert f_{t}\right\Vert _{\infty}\right\rceil $
so that we may assume that we may choose $f_{t}^{k_{min}(t)-1}=0$
as an estimator, and thus we recover the Euler-Maruyama method in
expectation 
\[
\mathbb{E}\left[y_{t+\eta}|y_{t}\right]=y_{t}+\eta f_{t}^{k_{max}}(y_{t})+\sqrt{\eta}\sigma_{t}Z_{t}.
\]

We will also make the following classical Lipschitzness assumptions:
\begin{assumption}
\label{assu:regularity_boundedness}For all $t$ and $k$, $Lip(f_{t}),Lip(f_{t}^{k})\leq L$.
\end{assumption}
We now bound the distance between $y_{t}$ and the Euler-Maruyama
discretization $x_{t}^{(\eta)}$ of the true flow
\[
x_{t+\eta}^{(\eta)}=x_{t}^{(\eta)}+\eta f_{t}(x_{t}^{(\eta)})+\sqrt{\eta}\sigma_{t}Z_{t}.
\]

\begin{thm}
\label{thm:ML-EM}Under Assumptions \ref{assu:scaling_law} and \ref{assu:regularity_boundedness},
for any step size $\eta>0$, error $\epsilon>0$, and time $T=i\eta>0$,
if we choose $k_{min}=-\left\lfloor \log_{2}c\right\rfloor $, $k_{max}=-\left\lfloor \log_{2}\left(\frac{2}{L}e^{L(T+\eta)}\epsilon\right)\right\rfloor $
and $p_{k}=\min\{C2^{-(1+\frac{\gamma}{2})k},1\}$ for some constant
$C$, we have $\mathbb{E}\left\Vert x_{T}^{(\eta)}-y_{T}\right\Vert ^{2}\leq\epsilon^{2}$
at an expected computational cost of at most 
\[
18\left[L^{3}T^{3}+\frac{LT}{2}\right]E_{\gamma}\left(\frac{ce^{L(T+\eta)}}{L\epsilon}\right)
\]
where
\[
E_{\gamma}(r)=\begin{cases}
\frac{1}{(1-2^{\frac{\gamma}{2}-1})^{2}}r^{2} & \gamma<2\\
r^{2}\left(3+\log_{2}r\right) & \gamma=2\\
\frac{2^{3(\gamma-2)}}{\left(2^{\frac{\gamma}{2}-1}-1\right)^{2}}r^{\gamma} & \gamma>2.
\end{cases}
\]
\end{thm}
\textbf{Harder than Monte Carlo (HTMC) regime ($\gamma>2$):} This
result is particularly relevant in the Harder than Monte Carlo (HTMC)
regime \cite{jacot2025_DLCPC}, when $\gamma>2$, where the computational
complexity ($\epsilon^{-\gamma}$) of solving the SDE is the same
as the complexity as a single evaluation of the best estimator!

As already discussed in Section \ref{subsec:Compute-Scaling-Analysis},
there is strong empirical evidence that DNNs follow scaling laws which
are ``flat enough'' to correspond to $\gamma$s that are far into
the HTMC regime. In Figure \ref{fig:estimate_gamma}, we estimate
$\gamma\approx2.5$ for the CelebA dataset (cropped and downscaled
to $64\times64$).

\textbf{Independence on step-size $\eta$: }Notice how the bound on
the compute required to reach an $\epsilon$ is essentially independent
of the step-size $\eta$ (to be precise, as $\eta\searrow0$, it decreases
to a finite value). This is because the probability that we evaluate
any one of the levels $f^{k}$ is proportional to $\eta$, therefore
the number of evaluations of each level $f^{k}$ remains constant
as $\eta\searrow0$. In this limit $y_{t}$ converges to some form
of Poisson jump process that approximates the original SDE $x_{t}$
with the same error and compute guarantees. This also implies that
there is no need to use more complex discretization scheme than the
EM method, because one can always take a smaller $\eta$ at no computational
cost (the only cost is that we need to sample the noise $Z_{t}$ and
add it, but when working with large DNNs, this computational cost
is negligible in comparison to the DNN evaluations).

\textbf{Choosing the probabilities $p_{k}$:} Theorem \ref{thm:ML-EM}
requires a very specific choice of probabilities, which requires knowledge
of the rate $\gamma,$ Lipschitz constant $L$ which are not really
accessible at first glance. Thankfully, it turns out that we have
a lot of flexibility in our choice of the $p_{k}$s, the proof can
easily adapted to show that if $p_{k}=C2^{-\beta k}$ for constant
$C$ and any exponent $\beta$ that lies in the range $(2,\gamma)$,
then the expected squared error will be $O(C^{-1}\epsilon^{2-\beta})$
with an $O(C\epsilon^{\gamma-\beta})$ expected computational cost,
so that by choosing $C\sim\epsilon^{-\beta}$, one can reach an $\epsilon$
error with an $O(\epsilon^{-\gamma})$ compute, recovering the right
rate. This means that we only have to tune one hyper-parameter, $C$,
to reach the optimal rate. Choosing $\beta=2$ or $\beta=\gamma$
also leads to the right rates up to some additional $\log\epsilon$
terms, but these are particularly straightforward to implement: $\beta=\gamma$
corresponds to choosing $p_{k}$ inversely proportional to the compute
time of $f^{k}$, which can easily be estimated.

Nevertheless, we also propose in the next section a method for learning
the $p_{k}$s with SGD to obtain as much computational gains as possible,
by not only obtaining the optimal rate, but also the optimal prefactor. 

\textbf{Choosing $k_{min}$ and $k_{max}$: }The choice of $k_{min}$
has very little impact on the final error. The choice of $k_{max}$
induces a lower bound on the minimal error that we can reach, since
the ML-EM method will always be less accurate than using only the
best estimator $f^{k_{max}}$ (though it can reach a similar error
much faster). In practice the choice of $k_{max}$ will mostly be
determined by computational constraints: what is the largest network
that can reasonably be trained on a certain compute budget.

\textbf{Exponential Constant:} The exponential term $e^{TL}$ emerges
naturally from the use of a Grönwall proof technique and also appears
in the classical EM method. It represents the fact that in the worst
case an error of size $\epsilon$ in the first few steps of the SDE
could get scaled up by $e^{TL}$ when we reach the final time $T$
(e.g. if $f(x)=Lx$). In diffusion models, since denoising acts as
a form of contraction rather than an expansion of the error, it is
reasonable to hope that this exponential blow-up will be naturally
avoided, and this seems to be what we observe in our experiments. 

Note that if one were instead in a setting where this exponential
blow-up is real, it might be advantageous to choose time-dependent
probabilities $p_{k}(t)$ that decrease in time, to make less errors
at times $t$ whose errors will be most impactful. We discuss a method
for doing so in the next section.

\subsection{Adaptive Method\label{subsec:Adaptive-Method}}

The question of how the errors from different times $t$ propagate
to the final time is very crucial in practice. ideally we would like
to adapt our estimation method to be more accurate at times $t$ where
errors contribute more to the final error. This can be achieved by
letting the probabilities $p_{k}$ and the max accuracy $k_{max}$
depend on time $t$. One could try to bound this error propagation
with some quantity and use it to choose $p_{k}(t)$ and $k_{max}(t)$.
Instead we take a very ``deep learning'' approach and learn the
optimal probabilities by minimizing the error with SGD directly. We
consider a simple time dependence 
\[
p_{k}(t)=\sigma\left(\alpha_{k}\log(t+\delta)+\beta_{k}\right)
\]
for parameters $\alpha_{k},\beta_{k}$ and small $\delta$ ($\delta=0.1$
in our experiments) and the sigmoid $\sigma$. Our goal is to find
the parameters $\alpha_{k},\beta_{k}$ that minimize the regularized
loss
\[
\mathcal{L_{\lambda}}(\alpha_{k},\beta_{k})=\mathbb{E}_{x_{T},Z_{t},B_{k}}\left\Vert x_{T}^{(\eta)}-y_{T}\right\Vert ^{2}+\lambda\sum_{i=0}^{T\eta^{-1}}p_{k}(i\eta)T_{k}
\]
where $T_{k}$ is the computational cost (either in FLOPs or in time)
of one evaluation of $f^{k}$, which can be easily estimated empirically.
The expectation is over the sampling of the starting point $x_{T}\sim\mathcal{N}(0,1)$
of the backward process and the noise $Z_{t}\sim\mathcal{N}(0,1)$
and Bernoullis $B_{k}(t)\sim Bernoulli(p_{k}(t))$ at each step.

There are two issues that make it hard to compute the gradient $\nabla\mathcal{L}_{\lambda}(\alpha_{k},\beta_{k})$:
we need to differentiate `through' the sampling of the Bernoulli random
variables, and on a more practical level, we cannot realistically
perform backpropagation through the whole SDE as it would require
keeping in memory all activations of every application of the network
(for all times $i\eta$ and all samples of $x_{T},Z_{t}$) which would
overshoot our memory budget. But these can be fixed with the right
techniques:

\textbf{Differentiating through Bernoullis:} For any function $f(B)$
of a Bernoulli random variable $B\sim Bernoulli(p)$, the derivative
of the expectation $\mathbb{E}[f(B)]$ w.r.t. its probability $p$
is $f(1)-f(0)$. Since 
\[
\mathbb{E}\left[f(B)\frac{B-p}{p(1-p)}\right]=pf(1)\frac{1-p}{p(1-p)}+(1-p)f(0)\frac{0-p}{p(1-p)}=f(1)-f(0),
\]
we can use $f(B)\frac{B-p}{p(1-p)}$ as an unbiased estimator for
$\frac{d}{dp}\mathbb{E}[f(B)]$. Now note that because we divide by
$p(1-p)$ which approaches zero as $p\approx0,1$ this estimator could
potentially have a lot of variance, but thankfully if $p$ is parametrized
as a sigmoid, as in our setting $p(t)=\sigma(\alpha\log(t+\delta)+\beta)$,
then by the chain rule, we have 
\begin{align*}
\partial_{\alpha}\mathbb{E}\left[f(B)\right] & =\mathbb{E}\left[f(B)\frac{B-p(t)}{p(t)(1-p(t))}\right]p(t)(1-p(t))\log(t+\delta)=\mathbb{E}\left[f(B)(B-p(t))\right]\log(t+\delta)\\
\partial_{\beta}\mathbb{E}\left[f(B)\right] & =\mathbb{E}\left[f(B)\frac{B-p(t)}{p(t)(1-p(t))}\right]p(t)(1-p(t))=\mathbb{E}\left[f(B)(B-p(t))\right]
\end{align*}
so that the estimates $f(B)(B-p(t))\log(t+\delta)$ and $f(B)(B-p(t))$
for the derivative w.r.t. $\alpha$ and $\beta$ can be expected to
have bounded variance as long as $f(B)$ and $\log(t+\delta)$ remain
bounded.

\textbf{Forward gradient computation instead of backpropagation:}
To avoid the memory cost of backpropagation, we instead rely on forward
propagation \cite{baydin2022_forward_gradients}, which allows us
to compute the scalar product $\nabla\mathcal{L}_{\lambda}^{T}v$
of the gradient $\nabla\mathcal{L}_{\lambda}$ with a vector $v$,
at a constant memory usage in time $i\eta$. The gradient $\nabla\mathcal{L}_{\lambda}$
can then be approximated by $\nabla\mathcal{L}_{\lambda}vv^{T}$ for
a random Gaussian vector $v\sim\mathcal{N}(0,I)$. This is again an
unbiased estimator since $\mathbb{E}_{v}\left[\nabla\mathcal{L}_{\lambda}vv^{T}\right]=\nabla\mathcal{L}_{\lambda}I=\nabla\mathcal{L}_{\lambda}$.

Putting everything together, we estimate the gradient $\nabla_{\alpha}\mathcal{L}_{\lambda}$
by 
\begin{align*}
 & \left\Vert x_{T}^{(\eta)}-y_{T}\right\Vert ^{2}\sum_{i=1}^{T\eta^{-1}}\left(B^{k}(i\eta)-p_{k}(i\eta)\right)\log(i\eta+\delta)\\
+ & \left(\nabla^{AD}\left\Vert x_{T}^{(\eta)}-y_{T}\right\Vert ^{2}\right)vv_{\alpha}^{T}\\
+ & \lambda\sum_{i=0}^{T\eta^{-1}}T_{k}p_{k}(i\eta)(1-p_{k}(i\eta))\log(i\eta+\delta)
\end{align*}
where $v$ is a random Gaussian vector of dimension $2(k_{max}-k_{min})$
(that is the same dimension as the $\alpha_{k},\beta_{k}$) and $v_{\alpha}$
is the first half of $v$ which corresponds to the $\alpha$s. For
the second term $\nabla^{AD}\left\Vert x_{T}^{(\eta)}-y_{T}\right\Vert ^{2}$
is the ``automatic differentiation'' gradient which treats the $B^{k}$
as if they were independent of $p_{k}$. Finally note how we use traditional
differentiation for the regularization term, since it does not suffer
from the two aforementioned challenges. Our estimate for the gradient
$\nabla_{\beta}\mathcal{L}_{\lambda}$ is obtained by removing the
$\log(i\eta+\delta)$ terms and replacing $v_{\alpha}$ by $v_{\beta}$
which is the second half of the vector $v$.

\section{Numerical Experiments\label{sec:Numerical-Experiments}}

\textbf{Training: }The experiments are performed on the CelebA dataset
\cite{liu_2015_CelebA}, cropped and rescaled to a size of $64\times64$.
This task was chosen as it matched the relatively limited compute
at our access (two GeForce RTX 2080 Ti).

We train a sequence of UNets $f^{1},f^{2},f^{3},f^{4},f^{5}$ of increasing
sizes, resulting in better and better approximations of the true score.
Our UNets have the following properties:
\begin{itemize}
\item At each level of the UNet, we divide the image dimension by two and
double the number of channels (starting from a certain ``base dimension'').
We have 4 levels so that at the ``bottom'' of the UNet has a $8\times8$
shape.
\item The filters are factored as the composition of a per-channel $3\times3$
convolution followed by a $1\times1$ convolution across channels.
\item There are $L_{1}$ residual layers at the bottom of the UNet, and
$L_{2}$ residual layers at the shallower scales, in both the downscaling
and the upscaling parts.
\item The four different networks have base dimensions $8,16,32,64$, bottom
depths $L_{1}=5,10,20,40$ and intermediate depths $L_{2}=2,3,5,7$
respectively.
\item Each of these networks were first trained separately on the usual
denoising loss, with Adam.
\end{itemize}
Note that it is now common practice to train multiple lower size models
to do hyper-parameter search before training the largest models, so
practitioners might already have access to a set of trained models
with a range of sizes and accuracies. And even if this is not the
case, the computational of cost of training the smaller models is
almost insignificant in comparison to the training cost of the larger
models.

\textbf{Generation:} For image generation, we followed the standard
DDPM procedure with a baseline of 1000 steps with a cosine noise schedule
\cite{nichol_2021_improved_DDPM}. We also applied clipping to the
predicted denoised image \cite{ho2020_DDPM}.

Since we do not have access to the true score we will use the largest
model $f^{5}$ with 1000 steps of generations as our `true generated
sample' and evaluate other generation methods in terms of how different
their generated images are from this true sample (with the same starting
noise $x_{T}$ and SDE noise $Z_{t}$).

For the baseline EM method, we try a range of number of steps: 250,
500, 750, 900, 1000 over the 5 networks $f^{1},\dots,f^{5}$. We can
see how changing the number of steps allows us tradeoff computation
time for a smaller error, and that the error seems to saturate a bit
before 1000 steps. Obviously when we approximate the 1000 steps $f^{5}$
generation with the same network with fewer steps, the error drops
very suddenly to zero as the number of steps approaches 1000, but
this very small error is misleading, because our actual goal is to
approach the generated images with the true score and we used our
largest UNet $f^{5}$ as a proxy for it. We therefore only focus on
errors above $10^{-3}$ as anything below this threshold is overfitting
to $f^{5}$ rather than approaching the true score.

For ML-EM we only used three models $\{f^{1},f^{3},f^{5}\}$. The
probabilities $p_{k}$ were chosen with three strategies: 
\begin{itemize}
\item ``Fixed probs.'' orange crosses: Taking $p_{k}=CT_{k}^{-1}$ is
the simplest method, since the average time computation time (or FLOPs)
$T_{k}$ of $f^{k}$ can easily be computed. As discussed in Section
\ref{sec:ML-EM} this is sufficient to obtain optimal rates. We then
vary $C$ to obtain a range of errors/times. With this method, the
probabilities are constant in time.
\item ``Fixed probs.'' green crosses: From our theory the optimal choice
of $p_{k}$ should be $p_{k}=C_{0}2^{-(1+\frac{\gamma}{2})k}=CT_{k}^{-(\frac{1}{\gamma}+\frac{1}{2})}$.
We estimate $\gamma=2.5$ (see Figure \ref{fig:estimate_gamma}) and
therefore choose $p_{k}=CT^{-0.9}$ over a range of $C$s. We do not
observe any significant differences between the two ``Fixed probs''
methods. 
\item ``Learned coeffs.'', blue dots: We optimize the $\alpha_{k},\beta_{k}$
parameters with 50 steps of SGD (as described in Section \ref{subsec:Adaptive-Method})
with a batch size of 300 and $\lambda=0.1$ for DDPMs and $\lambda=1.0$
for DDIM. We then obtain a range of errors/times by adding a delta
to the constant coefficients $\beta_{k}\leftarrow\beta_{k}+\Delta$
for $\Delta$ ranging from $-3.0$ to $3.0$. This method clearly
outperforms the ``Fixed probs.'' methods.
\end{itemize}
\textbf{GPU batching:} In GPUs, the compute time is typically only
linear in the number of function evaluations if these function evaluations
are batched. To take advantage of this we generate $N=200$ images
simultaneously and share the Bernoulli variables across the batch,
so that we either have to evaluate $f^{k}$ over the whole batch or
not at all, leading to a significant speedup.

However we do not use this trick when learning the $\alpha_{k},\beta_{k}$
with SGD, because we need our approximate gradient to concentrate,
and sharing Bernoullis breaks the independence leading to a higher
variance.

\section{Conclusion}

We introduce the ML-EM method for discretizing SDEs and ODEs that
is especially useful when the drift term lies in Harder than Monte
Carlo (HTMC) regime. This appears to apply to typical applications
of diffusion models, leading to a fourfold speedups in the time and
compute required to generate high quality images. The advantage of
ML-EM over EM should only increase for larger and more complex datasets,
and so one could expect tenfold speedups or more at the kind of scales
that are common in industry. This method can also be used in combination
with other methods for speeding up diffusion models, such as DDIM.

\bibliographystyle{plain}
\bibliography{main}

\pagebreak{}

\appendix

\begin{figure}
\centering
\includegraphics[scale=0.5]{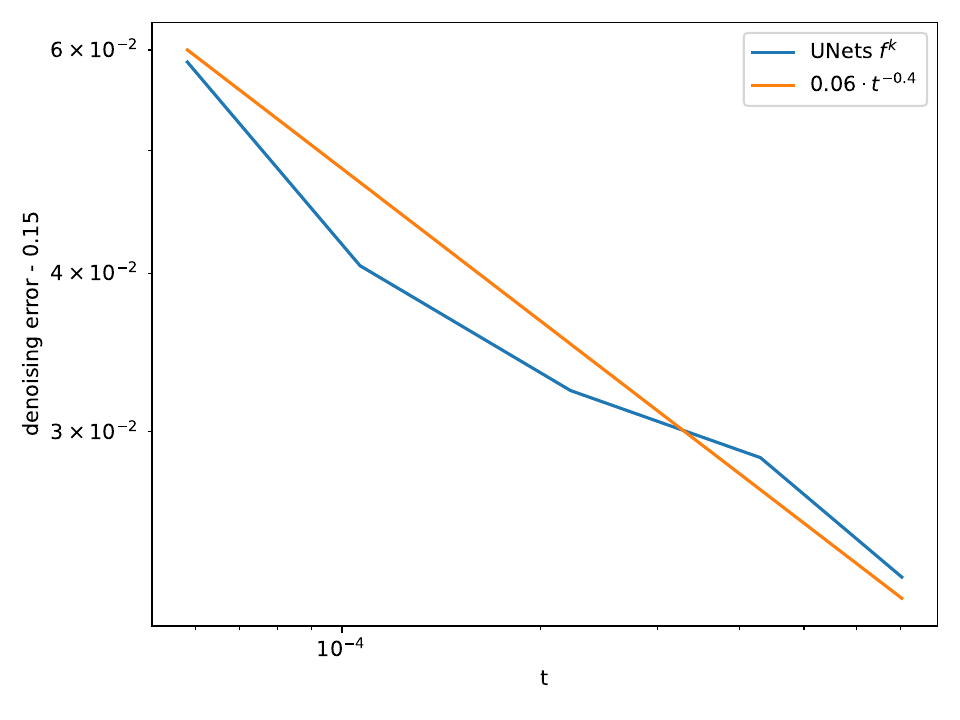}

\caption{\label{fig:estimate_gamma}Estimating $\gamma\approx2.5$: We plot
the denoising error $\epsilon$ minus $0.15$ against the evaluation
time for a range of UNets $f^{1},\dots,f^{5}$. The constant $0.15$
was chosen by hand to approximate the minimal denoising error (it
was chosen so that the set of points would align as closely as possible
to a line in the log-log plot). We see that on a log-log scale the
plot fits well with a $\epsilon\sim t^{-0.4}$ slope, which would
correspond to $\gamma=\frac{1}{0.4}=2.5$, which lies in the HTMC
regime ($\gamma>2$).}
\end{figure}

\section{DDPM/DDIM as approximate Euler-Maruyama methods \label{sec:appendix-comparison}}

In practice DDPMs (and DDIMs) are defined in terms of a sequence of
steps $\beta_{1},\dots,\beta_{M}$ (which are essentially equivalent
to time-dependent step-sizes $\eta_{t}$), which then define a forward
process
\[
y_{m}=\sqrt{1-\beta_{m}}y_{m-1}+\sqrt{\beta_{m}}Z_{m}
\]
for some noise $Z_{i}\sim\mathcal{N}(0,1)$. Defining $\alpha_{m}=1-\beta_{m}$
and $\bar{\alpha}_{m}=\alpha_{1}\cdots\alpha_{m}$, one can easily
prove that $y_{m}$ is Gaussian with
\[
y_{m}\sim\mathcal{N}\left(\sqrt{\bar{\alpha}_{m}}y_{0},(1-\bar{\alpha}_{m})I\right).
\]

We already see an approximation between $y_{m}$ and the continuous
process $x_{t}$ at time $t=\beta_{1}+\dots+\beta_{m}$ since $x_{t}\sim\mathcal{N}(e^{-\frac{t}{2}},(1-e^{-t})I)$
and $\bar{\alpha}_{m}\approx e^{-\beta_{1}-\dots-\beta_{m}}=e^{-t}$.

Given $\sigma_{m}=\sqrt{1-\bar{\alpha}_{m}}$ and $\epsilon_{m}(y_{m})=-\sigma_{m}\nabla\log\rho_{y_{m}}(y_{m})$
for $\rho_{y_{m}}$ the density of $y_{m}$, the DDPM backward process
is defined as
\[
y_{m-1}=\frac{1}{\sqrt{\alpha_{m}}}y_{m}-\frac{\beta_{m}}{\sqrt{\alpha_{m}}\sigma_{m}}\epsilon_{m}+\sqrt{\beta_{m}}\frac{\sigma_{m-1}}{\sigma_{m}}Z_{m},
\]
and the DDIM backward process is defined as
\[
\frac{y_{m-1}}{\sqrt{\bar{\alpha}_{m-1}}}=\frac{y_{m}}{\sqrt{\bar{\alpha}_{m}}}+\left(\sqrt{\frac{1-\bar{\alpha}_{m-1}}{\bar{\alpha}_{m-1}}}-\sqrt{\frac{1-\bar{\alpha}_{m}}{\bar{\alpha}_{m}}}\right)\epsilon_{m}(y_{m}).
\]

For the DDPM equivalence, observe that $\frac{1}{\sqrt{\alpha_{m}}}=\frac{1}{\sqrt{1-\beta_{m}}}=1+\frac{\beta_{m}}{2}+O(\beta_{m}^{2})$,
we then obtain similar approximations $\frac{\beta_{m}}{\sqrt{\alpha_{m}}\sigma_{m}}\approx\frac{\beta_{m}}{\sigma_{m}}$,
similarly $\frac{\sigma_{m-1}}{\sigma_{m}}\approx1$ (up to $O(\beta_{m}^{2})$
terms). 
\[
y_{m-1}-y_{m}\approx\beta_{m}\left[\frac{1}{2}y_{m}+\nabla\log\rho_{y_{m}}(y_{m})\right]+\sqrt{\beta_{m}}Z_{m}.
\]
This implies that the backward DDPM process is approximately equal
to an Euler-Maruyama approximation of the backward SDE $-dy_{t}=\left[\frac{1}{2}y_{t}+\nabla\log\rho_{y_{t}}(y_{t})\right]dt+dW_{t}$
with step-size $\beta_{m}$.

Let us first rewrite the DDIM formula, using the fact that $\bar{\alpha}_{m}=\alpha_{m}\cdot\bar{\alpha}_{m-1}$:
\[
y_{m-1}=\frac{y_{m}}{\sqrt{\alpha_{m}}}+\left(\sqrt{1-\bar{\alpha}_{m-1}}-\sqrt{\alpha_{m}^{-1}-\bar{\alpha}_{m-1}}\right)\epsilon_{m}(y_{m}).
\]
Approximating $\frac{1}{\sqrt{\alpha_{m}}}\approx1+\frac{\beta_{m}}{2}$
and taking a Taylor approximation of the function $\sqrt{x-\bar{\alpha}_{m-1}}$
around $x=1$, we obtain
\begin{align*}
y_{m-1}-y_{m} & \approx\frac{\beta_{m}}{2}y_{m}+\frac{1-\alpha_{m}^{-1}}{2\sqrt{1-\bar{\alpha}_{m-1}}}\epsilon_{m}(y_{m})\\
 & \approx\frac{\beta_{m}}{2}y_{m}+\frac{\beta_{m}}{2}\frac{\sigma_{m}}{\sigma_{m-1}}\nabla\log\rho_{y_{m}}(y_{m})\\
 & \approx\beta_{m}\left[\frac{1}{2}y_{m}+\frac{1}{2}\nabla\log\rho_{y_{m}}(y_{m})\right]
\end{align*}
which is the Euler approximation of the backward ODE $-\frac{dy_{t}}{dt}=\frac{1}{2}y_{t}+\frac{1}{2}\nabla\log\rho_{y_{t}}(y_{t})$.

\section{Proofs}
\begin{thm}
Under Assumptions \ref{assu:scaling_law} and \ref{assu:regularity_boundedness},
for any step size $\eta>0$, error $\epsilon>0$, and time $T=i\eta>0$,
if we choose $k_{min}=-\left\lfloor \log_{2}c\right\rfloor $, $k_{max}=-\left\lfloor \log_{2}\left(\frac{2}{L}e^{L(T+\eta)}\epsilon\right)\right\rfloor $
and $p_{k}=\min\{C2^{-(1+\frac{\gamma}{2})k},1\}$ for some constant
$C$, we have $\mathbb{E}\left\Vert x_{T}^{(\eta)}-y_{T}\right\Vert ^{2}\leq\epsilon^{2}$
at an expected computational cost of at most 
\[
18\left[L^{3}T^{3}+\frac{LT}{2}\right]E_{\gamma}\left(\frac{ce^{L(T+\eta)}}{L\epsilon}\right)
\]
where
\[
E_{\gamma}(r)=\begin{cases}
\frac{1}{(1-2^{\frac{\gamma}{2}-1})^{2}}r^{2} & \gamma<2\\
r^{2}\left(3+\log_{2}r\right) & \gamma=2\\
\frac{2^{3(\gamma-2)}}{\left(2^{\frac{\gamma}{2}-1}-1\right)^{2}}r^{\gamma} & \gamma>2.
\end{cases}
\]
\end{thm}
\begin{proof}
As a reminder, here is the formula for the MLMC-EM method

\[
y_{t+\eta}=y_{t}+\eta\sum_{k=k_{min}}^{k_{max}}\frac{B_{k}(t)}{p_{k}}\left[f_{t}^{k}(y_{t})-f_{t}^{k-1}(y_{t})\right]+\sqrt{\eta}\sigma_{t}Z_{t}.
\]
And note that since we chose $k_{min}=-\left\lfloor \log_{2}c\right\rfloor <-\log_{2}c+1$,
the estimator $f_{t}^{k_{min}-1}$ must have compute bounded by $c^{\gamma}2^{\gamma(k_{min}-1)}<1$
and therefore we take it to be the constant $0$ function.

We will track the evolution of the error in time (with the usual Grönwall's
Lemma strategy), splitting the error into a bias term $b_{t}=\left\Vert \mathbb{E}y_{t}-x_{t}^{(\eta)}\right\Vert $
and $v_{t}^{2}=\mathbb{E}\left\Vert y_{t}-\mathbb{E}y_{t}\right\Vert ^{2}$
where the expectation averages over the sampling of the $B_{k}(t)$,
not the $Z_{t}$ which we assume to be fixed (in other terms our analysis
conditions on the $Z_{t}$, which are shared between $y_{t}$ and
$x_{t}^{(\eta)}$). First we note that $b_{t+\eta}$ can be bounded
in terms of $b_{t}$ and $v_{t}$ 
\begin{align*}
b_{t+\eta} & =\left\Vert \mathbb{E}\left[y_{t}+\eta\sum_{k=k_{min}}^{k_{max}}\frac{B_{k}(t)}{p_{k}}\left[f_{t}^{k}(y_{t})-f_{t}^{k-1}(y_{t})\right]+\sqrt{\eta}\sigma_{t}Z_{t}\right]-\left(x_{t}^{(\eta)}+\eta f_{t}(x_{t}^{(\eta)})+\sqrt{\eta}\sigma_{t}Z_{t}\right)\right\Vert \\
 & =\left\Vert \mathbb{E}y_{t}-x_{t}^{(\eta)}+\eta\left(\mathbb{E}f^{k_{max}}(y_{t})-f(x_{t}^{(\eta)})\right)\right\Vert \\
 & \leq\left\Vert \mathbb{E}y_{t}-x_{t}^{(\eta)}\right\Vert +\eta\left\Vert \mathbb{E}f^{k_{max}}(y_{t})-\mathbb{E}f^{k}(y_{t})\right\Vert +\eta\left\Vert \mathbb{E}f(y_{t})-f(x_{t}^{(\eta)})\right\Vert \\
 & \leq b_{t}+\eta2^{-k_{max}}+\eta L\sqrt{b_{t}^{2}+v_{t}^{2}}\\
 & \leq(1+\eta L)b_{t}+\eta Lv_{t}+\eta2^{-k_{max}},
\end{align*}
where we used $\sqrt{b_{t}^{2}+v_{t}^{2}}\leq b_{t}+v_{t}$ in the
last inequality. 

On the other hand $v_{t+\eta}$ can be bounded in terms of $v_{t}$,
by relying on the conditional variance formula, conditioning on $y_{t}$:
\begin{align*}
v_{t+\eta}^{2} & =\mathbb{E}\left\Vert y_{t+\eta}-\mathbb{E}\left[y_{t+\eta}|y_{t}\right]\right\Vert ^{2}+\mathbb{E}\left\Vert \mathbb{E}\left[y_{t+\eta}|y_{t}\right]-\mathbb{E}y_{t+\eta}\right\Vert ^{2}\\
 & =\eta^{2}\sum_{k=k_{min}}^{k_{max}}\frac{1}{p_{k}}\mathbb{E}\left\Vert f_{t}^{k}(y_{t})-f_{t}^{k-1}(y_{t})\right\Vert ^{2}+\mathbb{E}\left\Vert y_{t}+\eta f^{k_{max}}(y_{t})-\mathbb{E}\left[y_{t}+\eta f^{k_{max}}(y_{t})\right]\right\Vert ^{2}\\
 & \leq9\eta^{2}\sum_{k=k_{min}}^{k_{max}}\frac{2^{-2k}}{p_{k}}+(1+\eta L)^{2}v_{t}^{2},
\end{align*}
where we used the fact that 
\[
\left\Vert f_{t}^{k}(x)-f_{t}^{k-1}(x)\right\Vert \leq\left\Vert f_{t}^{k}(x)-f_{t}(x)\right\Vert +\left\Vert f_{t}(x)-f_{t}^{k-1}(x)\right\Vert \leq2^{-k}+2^{-k+1}=3\cdot2^{-k}
\]
 and $\mathrm{Var}(g(X))\leq Lip(g)^{2}\mathrm{Var}(X)$ for the last
inequality.

We can unroll the recursive bound for $v_{t}$ into a direct bound
\begin{align*}
v_{i\eta}^{2} & \leq9\eta^{2}\sum_{j=0}^{i}(1+\eta L)^{2(i-j)}\sum_{k=k_{min}}^{k_{max}}\frac{2^{-2k}}{p_{k}}\\
 & \leq9\eta^{2}\frac{(1+\eta L)^{2i}}{1-(1+\eta L)^{-2}}\sum_{k=k_{min}}^{k_{max}}\frac{2^{-2k}}{p_{k}}\\
 & \leq\frac{9\eta}{2L}(1+\eta L)^{2(i+1)}\sum_{k=k_{min}}^{k_{max}}\frac{2^{-2k}}{p_{k}}\\
 & \leq\frac{9\eta}{2L}e^{2L(i+1)\eta}\sum_{k=k_{min}}^{k_{max}}\frac{2^{-2k}}{p_{k}}
\end{align*}
where we used
\begin{align*}
\frac{1}{1-(1+\eta L)^{-2}} & =\frac{(1+\eta L)^{2}}{(1+\eta L)^{2}-1}\leq\frac{(1+\eta L)^{2}}{2\eta L}.
\end{align*}
This in turn allows us to bound the bias term \textbf{$b_{t}$} directly
\begin{align*}
b_{i\eta} & \leq\eta L\sum_{j=0}^{i}(1+\eta L)^{i-j}v_{j}+\eta2^{-k_{max}}\sum_{j=0}^{i}(1+\eta L)^{i-j}\\
 & \leq\frac{3\sqrt{L}}{\sqrt{2}}\sqrt{1+2\eta L}\eta^{\frac{3}{2}}Li(1+\eta L)^{i}\sqrt{\sum_{k=k_{min}}^{k_{max}}\frac{2^{-2k}}{p_{k}}}+\eta\frac{(1+\eta L)^{i}}{1-(1+\eta L)^{-1}}2^{-k_{max}}\\
 & \leq\frac{3\sqrt{L}}{\sqrt{2}}\sqrt{\eta}(i\eta)e^{L(i+1)\eta}\sqrt{\sum_{k=k_{min}}^{k_{max}}\frac{2^{-2k}}{p_{k}}}+\frac{1}{L}e^{L(i+1)\eta}2^{-k_{max}}.
\end{align*}
To reach an $\epsilon$ error, we choose $k_{max}=-\left\lfloor \log_{2}\left(\frac{L}{2}e^{-L(i+1)\eta}\epsilon\right)\right\rfloor $
so that $\frac{1}{L}e^{L(i+1)\eta}2^{-k_{max}}\leq\frac{\epsilon}{2}$.
The $p_{k}$s are chosen as that $p_{k}=\min\{C2^{-(1+\frac{\gamma}{2})k},1\}$
for some constant $C$ so as to minimize both the sum $\sum_{k=k_{min}}^{k_{max}}\frac{2^{-2k}}{p_{k}}$
and the computational cost $c\sum_{k=k_{min}}^{k_{max}}p_{k}2^{\gamma k}$.
We will then choose a sufficiently large constant $C$ to guarantee
an $\epsilon$ error. 

Using the identity $(a+b)^{2}\leq2a^{2}+2b^{2}$ and $(1+\eta L)^{i}\leq e^{Li\eta}$,
we simplify the total expected squared error:
\begin{align*}
b_{i\eta}^{2}+v_{i\eta}^{2} & \leq\left(\frac{3\sqrt{L}}{\sqrt{2}}\sqrt{\eta}(i\eta)e^{L(i+1)\eta}C^{-\frac{1}{2}}\sqrt{\sum_{k=k_{min}}^{k_{max}}2^{(\frac{\gamma}{2}-1)k}}+\frac{\epsilon}{2}\right)^{2}+\frac{9\eta}{2L}e^{2L(i+1)\eta}C^{-1}\sum_{k=k_{min}}^{k_{max}}2^{(\frac{\gamma}{2}-1)k}\\
 & \leq9\eta\left[L(i\eta)^{2}+\frac{1}{2L}\right]e^{2L(i+1)\eta}C^{-1}\sum_{k=k_{min}}^{k_{max}}2^{(\frac{\gamma}{2}-1)k}+\frac{\epsilon^{2}}{2}.
\end{align*}
We therefore choose
\[
C=18\eta\left[L(i\eta)^{2}+\frac{1}{2L}\right]e^{2L(i+1)\eta}\sum_{k=k_{min}}^{k_{max}}2^{(\frac{\gamma}{2}-1)k}\epsilon^{-2}
\]
to obtain an expected squared error of $\epsilon^{2}$ at a computational
cost of at most
\begin{align*}
i\sum_{k=k_{min}}^{k_{max}}p_{k}c2^{\gamma k} & \leq iC\sum_{k=k_{min}}^{k_{max}}c2^{(\frac{\gamma}{2}-1)k}\\
 & =18\left[L(i\eta)^{3}+\frac{i\eta}{2L}\right]e^{2L(i+1)\eta}\left(\sum_{k=k_{min}}^{k_{max}}2^{(\frac{\gamma}{2}-1)k}\right)^{2}c\epsilon^{-2}
\end{align*}

The geometric sum $\sum_{k=k_{min}}^{k_{max}}2^{(\frac{\gamma}{2}-1)k}$
can be bounded in three cases:
\begin{align*}
\sum_{k=k_{min}}^{k_{max}}2^{(\frac{\gamma}{2}-1)k} & \leq\begin{cases}
\frac{1}{1-2^{\frac{\gamma}{2}-1}}2^{(\frac{\gamma}{2}-1)k_{min}} & \gamma<2\\
(k_{max}+1-k_{min}) & \gamma=2\\
\frac{2^{\frac{\gamma}{2}-1}}{2^{\frac{\gamma}{2}-1}-1}2^{(\frac{\gamma}{2}-1)k_{max}} & \gamma>2
\end{cases}\\
 & \leq\begin{cases}
\frac{1}{1-2^{\frac{\gamma}{2}-1}}c^{\frac{1}{\gamma}-\frac{1}{2}} & \gamma<2\\
\log_{2}\left(8\frac{c^{\frac{1}{\gamma}}e^{L(T+\eta)}}{L\epsilon}\right) & \gamma=2\\
\frac{2^{\gamma-2}}{2^{\frac{\gamma}{2}-1}-1}\left(\frac{L}{2}e^{-L(i+1)\eta}\epsilon\right)^{-(\frac{\gamma}{2}-1)} & \gamma>2.
\end{cases}
\end{align*}
This leads to the computational bound
\[
\sum_{k=k_{min}}^{k_{max}}p_{k}c2^{\gamma k}\leq18\left[(Li\eta)^{3}+\frac{Li\eta}{2}\right]\begin{cases}
\frac{1}{(1-2^{\frac{\gamma}{2}-1})^{2}}\left(\frac{c^{\frac{1}{\gamma}}e^{L(i+1)\eta}}{L\epsilon}\right)^{2} & \gamma<2\\
\left(\frac{c^{\frac{1}{\gamma}}e^{L(i+1)\eta}}{L\epsilon}\right)^{2}\log_{2}\left(8\frac{c^{\frac{1}{\gamma}}e^{L(i+1)\eta}}{L\epsilon}\right) & \gamma=2\\
\frac{2^{3(\gamma-2)}}{\left(2^{\frac{\gamma}{2}-1}-1\right)^{2}}\left(\frac{c^{\frac{1}{\gamma}}e^{L(i+1)\eta}}{L\epsilon}\right)^{\gamma} & \gamma>2.
\end{cases}
\]
\end{proof}

\end{document}